\documentclass{article} %
\usepackage{iclr2025_conference,times}

\usepackage{amsmath,amsfonts,bm}

\def\eqref#1{equation~\ref{#1}}

\def\1{\bm{1}}

\DeclareMathAlphabet{\mathsfit}{\encodingdefault}{\sfdefault}{m}{sl}
\SetMathAlphabet{\mathsfit}{bold}{\encodingdefault}{\sfdefault}{bx}{n}

\newcommand{\R}{\mathbb{R}}

\usepackage{amsthm}
\usepackage{hyperref}
\usepackage{url}
\usepackage{wrapfig}
\usepackage{thm-restate}
\usepackage{xcolor}
\usepackage{xspace}
\usepackage{graphicx}
\usepackage{caption,subcaption}

\newtheorem*{proposition*}{Proposition}

\title{Conformal Transformations for Symmetric Power Transformers}

\author{Saurabh Kumar\thanks{Denotes equal contribution.}$\text{ }$\thanks{Work completed during internship at Manifest AI. Correspondence to \texttt{szk@stanford.edu}.} \\
Stanford University \\
\And
Jacob Buckman$^*$ \\
Manifest AI \\
\And 
Carles Gelada \\
Manifest AI \\
\And
Sean Zhang \\
Manifest AI
}

\iclrfinalcopy %
\begin{document}

\maketitle

\begin{abstract}
Transformers with linear attention offer significant computational advantages over softmax-based transformers but often suffer from degraded performance. The symmetric power (sympow) transformer, a particular type of linear transformer, addresses some of this performance gap by leveraging symmetric tensor embeddings, achieving comparable performance to softmax transformers. However, the finite capacity of the recurrent state in sympow transformers limits their ability to retain information, leading to performance degradation when scaling the training or evaluation context length. To address this issue, we propose the \textit{conformal-sympow} transformer, which dynamically frees up capacity using data-dependent multiplicative gating and adaptively stores information using data-dependent rotary embeddings. Preliminary experiments on the LongCrawl64 dataset demonstrate that conformal-sympow overcomes the limitations of sympow transformers, achieving robust performance across scaled training and evaluation contexts.
\end{abstract}

\section{Introduction}

Transformers with softmax attention~\citep{vaswani2017attention} have computational cost that is quadratic in context length. A popular solution is to remove the exponential in the softmax, resulting in a linear attention~\citep{katharopoulos2020transformers,choromanski2020rethinking}. Transformers with linear attention admit a corresponding recurrent formulation enabling linear time inference and a chunked formulation with sub-quadratic training cost. However, while linear transformers enjoy practical speedups, they suffer from degraded performance relative to softmax transformers~\citep{kasai2021finetuning}.

To bridge the performance gap with softmax transformers, recent work has proposed a variant of linear transformers called symmetric power (sympow) transformers~\citep{buckman2024sympow}. Sympow transformers embed queries and keys using an embedding function based on the theory of symmetric tensors. \citet{buckman2024sympow} demonstrates that sympow achieves comparable performance to a softmax transformer baseline while maintaining a tractably small recurrent state size. 

While the recurrent formulation enables efficient training and inference, a linear transformer's recurrent state is fundamentally constrained by its finite-dimensional representation. Additionally, in the attention formulation, a linear attention mechanism limits the class of attention score distributions, typically favoring more diffuse distributions relative to softmax attention. As a result, linear transformers may struggle in tasks that require synthesizing information from long contexts. Our experiments confirm that sympow transformers exhibit degraded performance both when increasing the training context and when evaluating on contexts longer than those seen during training.

In this paper, we introduce mechanisms that enable symmetric power transformers to manage their constrained capacity more effectively.
We first apply data-dependent multiplicative gating developed in prior work~\citep{dao2024transformers} which erases information in the recurrent state, freeing up capacity for new information. 
We then introduce a novel approach for learning data-dependent rotary embeddings which iteratively applies dynamically chosen rotations to the recurrent state. 
Data-dependent rotations enable the model to adaptively determine where to store information in embedding space. 
The combination of data-dependent gating and rotary embeddings forms a learned \textit{conformal linear transformation} to the recurrent state. We refer to the resulting conformal symmetric power transformer as \textit{conformal-sympow}. Our experiments on the LongCrawl64 dataset~\citep{buckman2024longcrawl} demonstrate that conformal-sympow overcomes the limitations of sympow transformers, achieving robust performance across scaled training and evaluation contexts.

\section{Background}\label{sec:sympow_detailed}
In this section, we provide a review of symmetric power transformers~\citep{buckman2024sympow}. We also review rotary embeddings~\citep{su2024roformer} which are an essential component of the conformal transformations discussed in this paper. We demonstrate that rotary embeddings are compatible with sympow transformers. 

\subsection{Symmetric Power Transformers}
In a linear transformer, the exponential in a softmax transformer is replaced with a kernel function along with an associated feature map. A symmetric power transformer is a particular type of linear transformer which uses kernel function \textbf{k}$(v, w) = (v^Tw)^p$ where $p$ is referred to as the symmetric power. The corresponding feature map $\phi^p : \R^d \to \R^D$ satisfies the following: for input $v \in \mathbb{R}^d$, $\phi(v) \in \mathbb{R}^D$ contains the same information as ${v\otimes \cdots \otimes v} \in \mathbb{R}^{d^p}$, repeatedly taking the tensor product $p$ times. It does so much more efficiently because it removes a lot of symmetry in the tensor product (hence the name symmetric power). Thus $D << d^p$.

In a sympow transformer, the inputs to an attention layer are sequences of $Q_i, K_i, V_i \in \R^d$ of queries, keys, and values, where $i$ ranges from $1$ to the sequence length $t$. The outputs are a sequence $Y_i\in \R^d$. In the \textit{attention formulation} of a sympow transformer with power $p$, the formula for the output vectors is:
$$
Y_i = \sum_{j=1}^i A_{ij} V_j \qquad A_{ij}  = \frac{B_{ij}}{\sum_{k=1}^i B_{ik}} \qquad B_{ij} = (Q_i^T K_j)^p
 \qquad \text{(sympow)}
$$
We refer to $A_{ij}$ as the attention scores and $B_{ij}$ as the pre-attention scores. 

In practice, it is important that the symmetric power $p$ is even because that guarantees that each $B_{ij}$ is non-negative, which makes the set of attention scores $A_{i1}, \cdots, A_{ii}$ a valid probability distribution. In turn, this makes the outputs $Y_i$ a convex combination of the value vectors $V_1, \cdots, V_i$.

A key feature of linear transformers is that the exact same outputs $Y_i$ can be computed via a \textit{recurrent formulation}. In the case of sympow transformers, doing so involves the feature map $\phi^p : \R^d \to \R^D$, which is the symmetric power embedding function. 
Using this embedding function, we can write the recurrent equations:
$$
Y_{i} = \frac{S_i \phi^p(Q_i)}{Z_i \phi^p(Q_i)} \qquad Z_i = Z_{i-1} + \phi^p(K_i)^T \qquad S_i = S_{i-1} + V_i \phi^p(K_i)^T
$$
where $Z_0 \in \R^{1 \times D}$ and $S_0 \in \R^{d \times D}$ are $0$ vectors in their respective spaces, and the tuple $(S_i, Z_i) \in \R^{(d+1) \times D}$ is the recurrent state that the sympow transformer stores, allowing for linear time inference. 

The equivalence between the attention and recurrent formulations arises from the fact that the embedding function $\phi^p$ satisfies the following property: for any two vectors $v, w \in \R^d$, $\phi^p(v)^T \phi^p(w) = (v^T w)^p$. The attention and recurrent forms give rise to a variety of algorithms for training linear transformers, which allow for subquadratic training cost and linear time inference.

\subsection{Compatibility of Rotary Embeddings with Sympow}\label{sec:rotary_compatibility}

Rotary embeddings~\citep{su2024roformer} are a type of positional encoding which encode time information by rotating the keys and queries by an amount proportional to their corresponding timestep. In particular, the rotation matrix $R\in \R^{d\times d}$ tells us how much we want to rotate every timestep, so that:
$$
Q'_{i} = R^i Q_i \qquad K'_j = R^j K_j
$$

We now show that rotary embeddings are compatible with sympow transformers. Specifically, after applying rotary embeddings, the attention scores remain a distribution. Further, the attention formulation with rotary embeddings has an equivalent recurrent formulation.

In sympow transformers, the pre-attention is changed to:
$$
B_{ij} = \left({Q'_{i}}^T K'_j \right)^p = \left({Q_{i}}^T (R^{i-j})^T K_j \right)^p
\qquad \text{(sympow rotary)}
$$
Importantly, rotary embeddings are compatible with the attention formulation when $p$ is even: each $B_{ij}$ is non-negative which makes $A_{i1}, ..., A_{ii}$ a valid probability distribution.

It is evident that the effect of rotation of the embeddings is \textit{relative} because it modulates interaction between $Q_i$ and $K_j$ depending only on the time difference $i-j$.

The rotation matrix $R$ is constructed in a particular way. We start with rotation rates $\theta_1, \theta_2, \cdots, \theta_{\frac d 2}$ distributed in the range $(0, 2\pi)$. Specifically, $\theta_i = \frac{2\pi}{N^{\frac{2(i-1)}{d}}}$, where $N$ is the maximum document size. The vector $\theta = (\theta_1, \theta_2, ..., \theta_{\frac{d}{2}})$ contains these rotation rates. Then, the rotation matrix is the following block diagonal matrix:
\[
R(\theta) = \begin{pmatrix}
\cos(\theta_1) & -\sin(\theta_1) & \cdots & 0 & 0 \\
\sin(\theta_1) & \cos(\theta_1)  & \cdots & 0 & 0 \\
\vdots & \vdots & \ddots & \vdots & \vdots \\
0 & 0 & \cdots & \cos(\theta_{d/2}) & -\sin(\theta_{d/2}) \\
0 & 0 & \cdots & \sin(\theta_{d/2}) & \cos(\theta_{d/2})
\end{pmatrix},
\]

When multiplying a query or key vector by this rotation matrix, each pair of dimensions indexed by $2j - 1$ and $2j$ for $j \in \{1, 2, ..., \frac{d}{2} \}$ is rotated by a different amount $\theta_j$. Rotating each pair by a different angle helps break symmetry and increases the expressiveness of the positional encodings.

A computational advantage of using rotation matrices of this form is that there is an efficient way to compute $R(\theta)^k = R(k \theta)$, which massively simplifies the cost of computing all the $Q'_i$ and $K'_j$. We include a derivation of this fact in Appendix~\ref{app:derivations}.

Now we want to find the recurrent formulation of rotary embeddings with sympow transformers. A simple way we can do that is by including one extra vector in the recurrent state which is now a tuple $(S, Z, \mu)$, where $\mu \in \R^{\frac d 2}$. The recurrent equations are given by

$$
Z_i = Z_{i-1} + \phi^p(R(\mu_i) K_i)^T \qquad S_i = S_{i-1} + V_i \phi^p(R(\mu_i) K_i)^T \qquad \mu_i = \mu_{i-1} + \theta
$$

Note we rotate the keys by $R(\mu_i)$ before using them.

Given $S_i$ and $Z_i$, the outputs are the same as before, except that we rotate the queries by $R(\mu)$ before using them:
$$
Y_{i} = \frac{S_i \phi^p( R(\mu_i) Q_i)}{Z_i \phi^p( R(\mu_i) Q_i)} 
$$

\begin{restatable}{proposition}{rotaryequivalence}\label{prop:rotaryequivalence}
When using rotary embeddings with sympow transformers, the attention formulation of the output $Y_i$ at time step $i$ is equivalent to its recurrent formulation. Specifically,
\[
\begin{alignedat}{2}
&Y_i = \sum_{j=1}^i \frac{(Q_i'^T K_j')^p V_j}{\sum_{k=1}^i (Q_i'^T K_k')^p} \quad
\text{ is equivalent to }
\quad
&&Y_i = \frac{S_i \phi^p(R(\mu_i) Q_i)}{Z_i \phi^p(R(\mu_i) Q_i)}.
\end{alignedat}
\]
\end{restatable}
The proof is in Appendix~\ref{app:derivations}.

\citet{buckman2024sympow} demonstrates that sympow achieves comparable performance to a softmax transformer baseline while maintaining a tractably small recurrent state size. However, our experiments in Section~\ref{sec:experiments} show that sympow's performance degrades relative to a softmax transformer at longer context lengths. In this paper, we propose mechanisms to help sympow better manage its limited state capacity, which we hypothesize is a key factor behind this performance degradation.

\section{Conformal Transformations}\label{sec:conformaltransformations}

In this section, we propose learning conformal transformations to the sympow transformer's recurrent state in order to better manage its limited capacity. A \textit{conformal linear transformation} is a type of linear transformation that preserves angles between vectors while allowing uniform scaling of lengths. Mathematically, a conformal linear transformation in \( n \)-dimensional Euclidean space can be expressed as

$$\mathbf{T}(\mathbf{x}) = s \mathbf{R} \mathbf{x}$$

where \( s > 0 \) is a scalar representing the scaling factor, \( \mathbf{R} \) is an orthogonal matrix (\( \mathbf{R}^T \mathbf{R} = \mathbf{I} \)), and \( \mathbf{x} \) is the input vector.

To update the recurrent state of a sympow transformer, we right multiply the state with a conformal transformation before adding new information:

\begin{equation}\label{eq:conformal_state_update}
S_{i} = S_{i-1} (s\mathbf{R}) + V_i \phi^p(K_i)^T
\end{equation}

In this paper, we consider conformal transformations for which the scalar satisfies $s < 1$ and the matrix $\mathbf{R}$ is a rotation matrix, leveraging the benefits of gating and rotary embeddings. Gating (multiplying the state by a scalar between $0$ and $1$), serves to erase information that is no longer needed and applying rotations serves to store information in the state more efficiently. In Section~\ref{sec:gating}, we apply data-dependent gating developed in prior work to sympow transformers. In Section~\ref{sec:rotary}, we introduce a novel approach to learn data-dependent rotations. We unify these concepts to form the conformal-sympow transformer in Section~\ref{sec:conformalsympow}.

\subsection{Data Dependent Gating}\label{sec:gating}

The basic idea of gating is that at each time step, the state matrix $S\in \R^{d \times D}$ will be discounted by a scalar $\gamma \in [0, 1]$. Discounting the state ``erases" past information stored in the state. This technique has been used extensively throughout the linear transformer literature~\citep{peng2021random,mao2022fine,katsch2023gateloop}. One common approach to implement gating is to pick a fixed gating value for each head, usually using a range of large and small $\gamma$ for different heads to allow the model to keep track of short and long term interactions. The gating values can also be learnable parameters or even data-dependent values, as has been thoroughly explored in prior work~\citep{dao2024transformers,sun2024you,gu2023mamba,yang2023gated}. 

In this paper, we apply the technique proposed in ~\citet{dao2024transformers} and used in \citet{peng2021random}, \citet{sun2024you}, and \citet{beck2024xlstm} to symmetric power transformers. Scalar discount values for gating are computed in a data-dependent manner using parameters $W_\gamma$ in each attention head. The discount value at timestep $i$ is $\gamma_i = \sigma(W_\gamma X_i)$ where $\sigma$ refers to the sigmoid function, $W_\gamma \in \mathbb{R}^{d \times 1}$ and $X_1, X_2, ..., X_i, ...$ is the input sequence from which the keys, queries, and values are computed (e.g. $K_i = W_K X_i$). When using symmetric power attention with power $p$, the recurrent state update is simply
$$
Z_i = \gamma_i Z_{i-1} + \phi^p(K'_i)^T \qquad S_i = \gamma_i S_{i-1} + V_i \phi^p(K'_i)^T \qquad \mu_i = \mu_{i-1} + \theta
$$

Recall that $K'_i = R^i K_i = R(\mu_i) K_i$ uses the notation for computing and applying rotary positional embeddings introduced in Section~\ref{sec:rotary_compatibility}. $R(\theta)$ (shorthand $R$) is the rotation matrix used to compute rotary positional embeddings at each time step using a vector of rotation rates $\theta$. The rotation applied to keys at position $i$ is $R^i$ which is equivalent to $R(\theta)^i = R(i\theta) = R(\mu_i)$.

To write the attention formulation we define $b_{ij} = \Pi_{k=j+1}^i \gamma_m$. Then, in the attention formulation, the preattention scores become
$$
B_{ij} = b_{ij} \; ( {Q'_i}^T K'_j)^p
\qquad \text{(sympow+gating)}
$$

\begin{restatable}{proposition}{gatingequivalence}\label{prop:gatingequivalence}
When applying scalar gating to sympow transformers, the attention formulation of the output $Y_i$ at time step $i$ is equivalent to its recurrent formulation. Specifically,
$$
Y_i = \sum_{j=1}^i \frac{b_{ij} \left(Q_i'^T K'_j\right)^p V_j}{\sum_{k=1}^i b_{ik} (Q_i'^T K_k')^p} \qquad \text{is equivalent to}  \qquad Y_i = \frac{S_i  \phi^p(Q'_i)}{Z_i \phi^p(Q'_i)}
$$
\end{restatable}
The proof is in Appendix~\ref{app:derivations}.

\subsection{Data Dependent Rotations}\label{sec:rotary}

We now introduce an approach to learn rotation rates in rotary positional embeddings. For a review of rotary embeddings and their compatibility with sympow transformers, see Section~\ref{sec:rotary_compatibility}. 

There are many possible ways to learn rotation rates. For example, the model could independently decide how much to rotate each pair of dimensions in the query or key vector. In this paper, we start with the fixed rotation rates $\theta_1, \theta_2, ..., \theta_{\frac{d}{2}}$ where $\theta_i = \frac{2\pi}{N^{\frac{2(i-1)}{d}}}$, and we equip the model with the ability to uniformly scale the rotation rates. Similar to the data-dependent gating approach, we add parameters $W_{\beta}$ to each attention layer. At each time step $i$, the attention layer outputs a scalar $\beta_i = 1 + \text{tanh}(W_\beta X_i)$ which it uses to scale the fixed vector $\theta$ that rotary embeddings typically apply. Recall that $X_1, X_2, ..., X_i, ...$ is the input sequence, so the scalar $\beta_i$ is data dependent. In the recurrent formulation, this produces the equation
$$
\mu_i = \mu_{i-1} + \beta_i \theta \quad \text{ where } \quad \beta_i = 1 + \text{tanh}(W_\beta X_i)
$$

Intuitively, since the range of $1 + \text{tanh}(\cdot)$ is $(0, 2)$, the model can decide to either speed up the rotation rates by up to $2$ times or reduce the rotation rates until effectively zero rotation is applied. While learned gating affects how information in the state is erased, learned rotations affect where information gets stored in embedding space. 

To write the attention formulation we define $c_{ij} = \sum_{k=j+1}^i \beta_i$. 
Then, in the attention formulation, the preattention scores become
$$
B_{ij} = \left(Q_i R(c_{ij}\theta) K_j \right)^p \qquad \text{(sympow learned rotary)}
$$

\subsection{Conformal Sympow}\label{sec:conformalsympow}

We combine data dependent gating and data dependent rotary embeddings into the following manner of computing preattention scores:
$$
B_{ij} = b_{ij} \left(Q_i R(c_{ij}\theta) K_j \right)^p \qquad \text{(conformal-sympow)}
$$

The corresponding recurrent state update is
\begin{equation}\label{eq:conformal_in_practice}
Z_i = \gamma_i Z_{i-1} + \phi^p(R(\mu_i) K_i)^T \qquad S_i = \gamma_i S_{i-1} + V_i \phi^p(R(\mu_i) K_i)^T \qquad \mu_i = \mu_{i-1} + \beta_i \theta
\end{equation}

As written, the recurrent state update above does not match the form of a conformal transformation applied to the state as in Equation~\ref{eq:conformal_state_update}. While in Equation~\ref{eq:conformal_state_update}, the state is right multiplied by a rotation matrix, in the above update in Equation~\ref{eq:conformal_in_practice}, the rotation appears when transforming the key vectors $K'_i = R(\mu_i)K_i$. There is an equivalent form to the above state update which does right multiply the state by a conformal transformation. The construction of this form uses the following result.

\begin{restatable}{proposition}{conformalequivalence}\label{prop:conformalequivalence}
For any vector $k \in \mathbb{R}^d$, if $P\in \R^{d\times d}$ is a rotation matrix, then there exists another rotation matrix $\bar P \in \R^{D \times D}$ s.t.
$$
\phi^p( P k) = \bar P \phi^p(k)
$$
\end{restatable}

The proof is in Appendix~\ref{app:derivations}. Using the above result, the recurrent state update can be written as follows:

\begin{equation}\label{eq:conformal_in_theory}
Z_i = Z_{i-1} (\gamma_i \bar{R}(\theta, \beta_i)) + \phi^p(K_i)^T \qquad S_i = S_{i-1} (\gamma_i \bar{R}(\theta, \beta_i)) + V_i \phi^p(K_i)^T 
\end{equation}
where $\bar{R}(\theta, \beta_i)$ is a rotation matrix that depends on the fixed rotation rates $\theta$ and the scalar $\beta_i$. 

In practice, the recurrent state update (\ref{eq:conformal_in_practice}) is more straightforward to implement, but it is equivalent to the conformal recurrent state update (\ref{eq:conformal_in_theory}). We now have the conformal-sympow transformer in both attention and recurrent formulations.

\section{Experiments}\label{sec:experiments}

\begin{figure}[t] %
    \centering
    \begin{subfigure}[b]{0.45\textwidth}
        \centering
        \includegraphics[width=\textwidth]{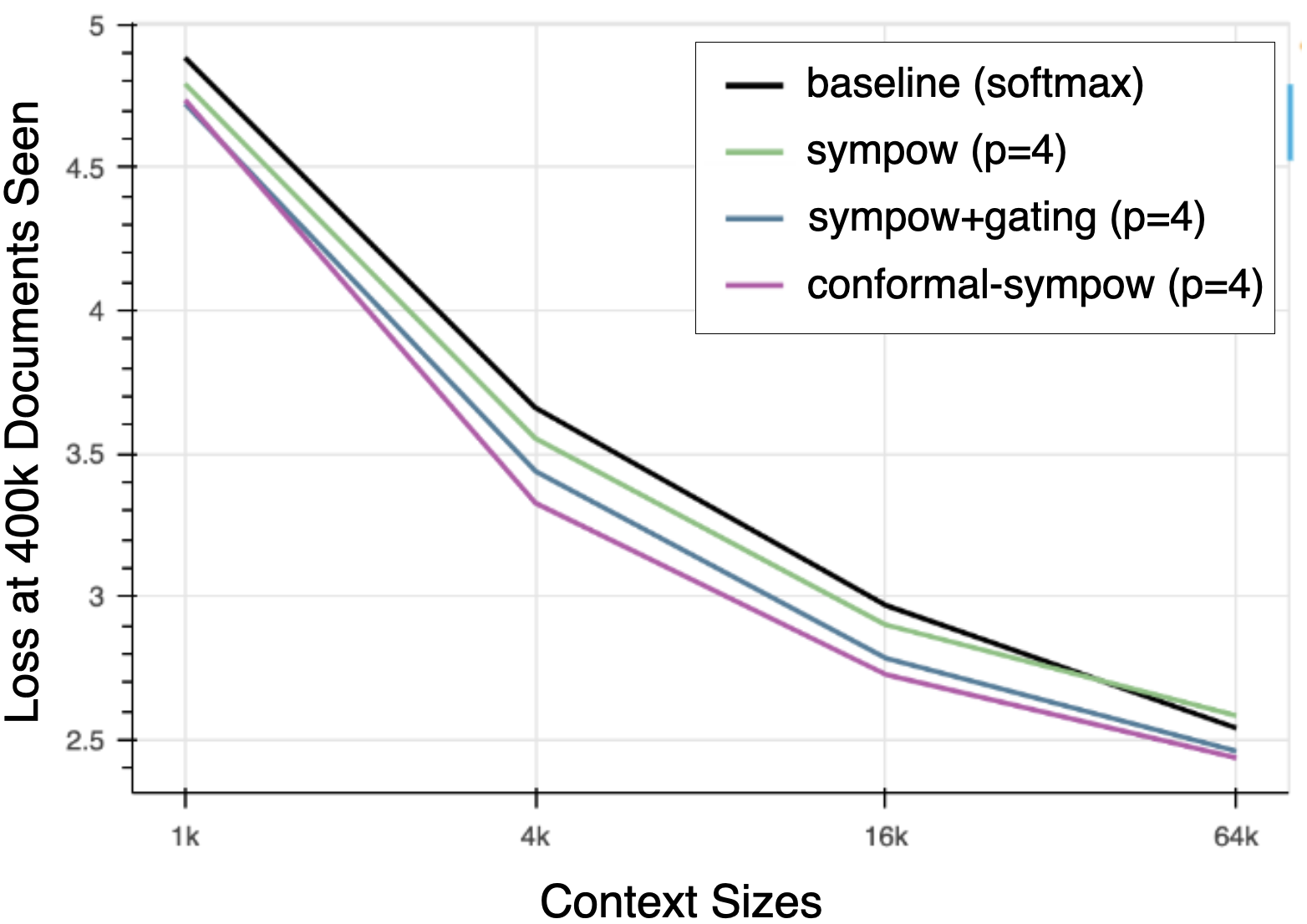}
        \caption*{(a) $p=4$}
    \end{subfigure}
    \hfill
    \begin{subfigure}[b]{0.45\textwidth}
        \centering
        \includegraphics[width=\textwidth]{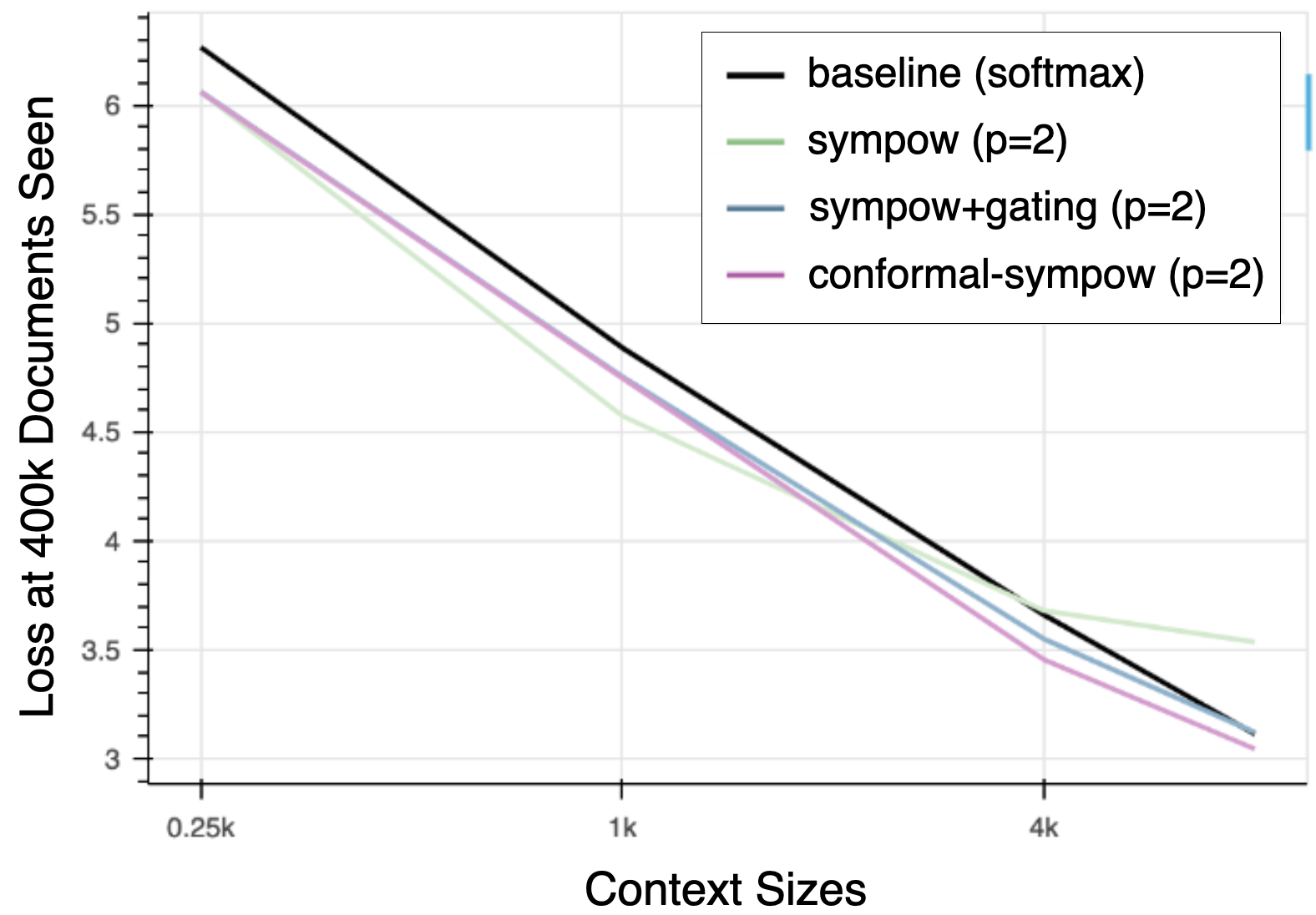}
        \caption*{(b) $p=2$}
    \end{subfigure}
    \caption{The training performance of sympow degrades relative to a softmax transformer baseline as the context size grows. Sympow with data-dependent gating (sympow+gating)
    closes this performance gap. Training performance further improves when adding data-dependent rotations with conformal-sympow. In contrast to sympow, conformal-sympow does not suffer from the degraded scaling of training context, either when (a) $p=4$ or (b) $p=2$.}  
    \vspace{-0.3cm}
\label{fig:conformal_sympow_vs_softmax_train}
\end{figure}

There are two main goals of our experiments: (1) evaluate sympow transformers on longer training and evaluation contexts than in prior work, and (2) determine the effectiveness of conformal transformations in closing any performance gap that arises at longer contexts.

For all experiments, we used the LongCrawl64 dataset~\citep{buckman2024longcrawl}, with a batch size of 524288 tokens. We used a transformer architecture that is similar to the 124M-parameter GPT-2 architecture but with rotary positional encoding and an additional layernorm after input embeddings. We performed optimization in bf$16$ mixed-precision using Adam with learning rate $.0006$ and no scheduling. Each model was trained on a node of $8$ H100s. Since the goal of our experiments is to understand the performance of sympow and conformal-sympow, we implemented the attention formulation of sympow (with quadratic cost) instead of the more efficient chunked formulation, which would require writing custom CUDA kernels. Additional results are in Appendix~\ref{app:experiments}.

\subsection{Are Sympow Transformers Robust to Context Scaling?} 

Prior work introducing sympow transformers demonstrated that sympow with a tractably small recurrent state (power $p=4$) closes the performance gap with a softmax transformer baseline at context size $4096$ on the LongCrawl64 dataset~\citep{buckman2024sympow}.
We ran experiments with the same setup but scaled the training context size from $1024$ up to $65,536$ when using sympow with $p=4$. 
We additionally ran experiments from training context size $256$ up to $8,192$ using sympow with $p=2$. Using $p=2$ results in a smaller recurrent state size ($39$ MB) than when using $p=4$ ($14$ GB). \citet{buckman2024sympow} consider state sizes under $80$ GB (the memory capacity of A100 and H100 GPUs) as tractable.

Our results in Figure~\ref{fig:conformal_sympow_vs_softmax_train} demonstrate that symmetric power transformers suffer degraded training performance compared to the softmax baseline as the training context size grows. While the effect is much more pronounced for $p=2$ model than $p=4$, solving poor scaling of the training context size will be important even for models with larger state sizes when scaling to millions of tokens in the training context. We further evaluate the ability of a trained model to make predictions at different evaluation context lengths, including ones longer than the ones used during training. In Figure~\ref{fig:sympow_eval}(a), we see that sympow does not generalize well beyond the training context size.

\begin{figure}[t] %
    \centering
    \begin{subfigure}[b]{0.45\textwidth}
        \centering
        \includegraphics[width=\textwidth]{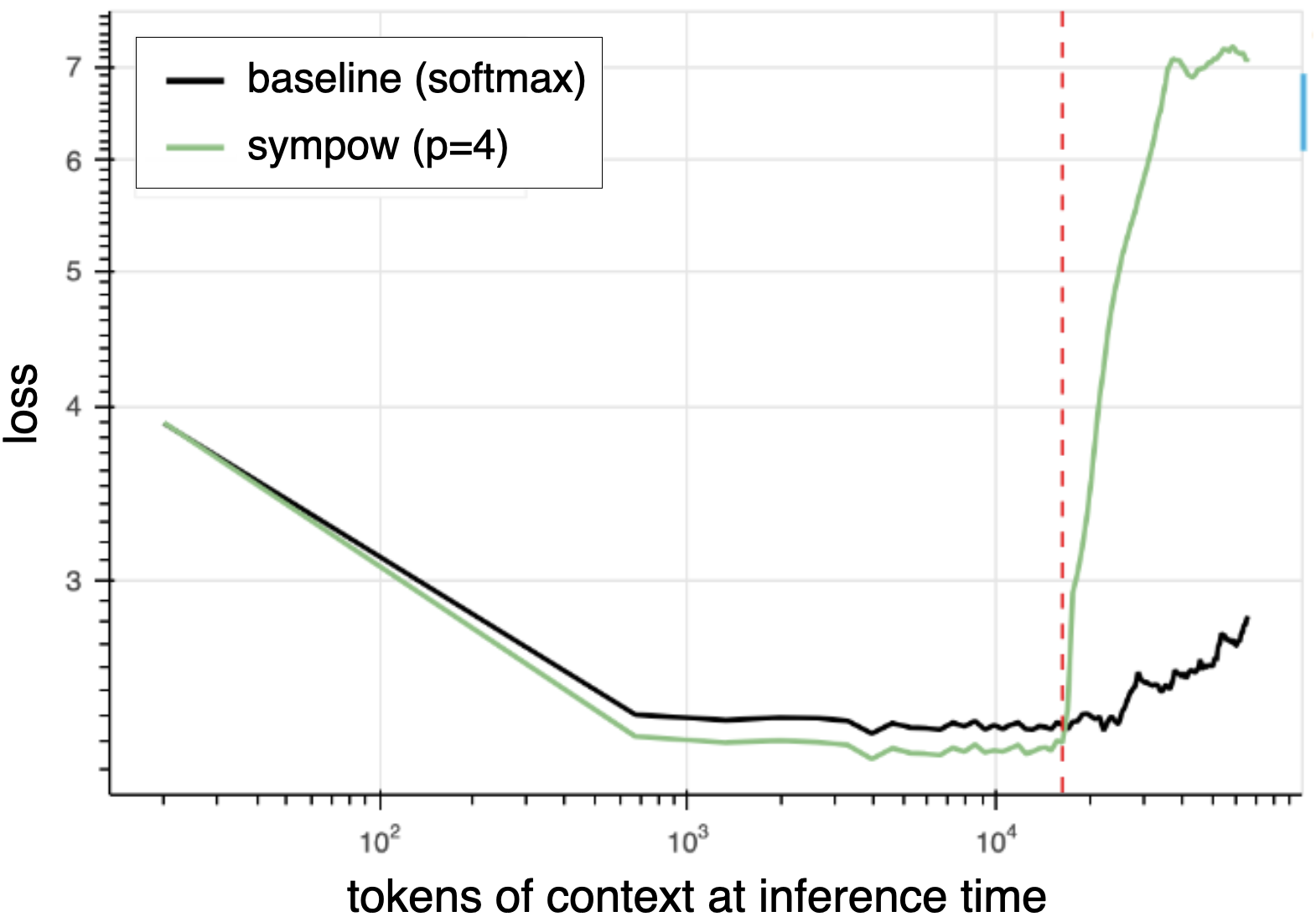}
        \caption*{(a)}
    \end{subfigure}
    \hfill
    \begin{subfigure}[b]{0.45\textwidth}
        \centering
        \includegraphics[width=\textwidth]{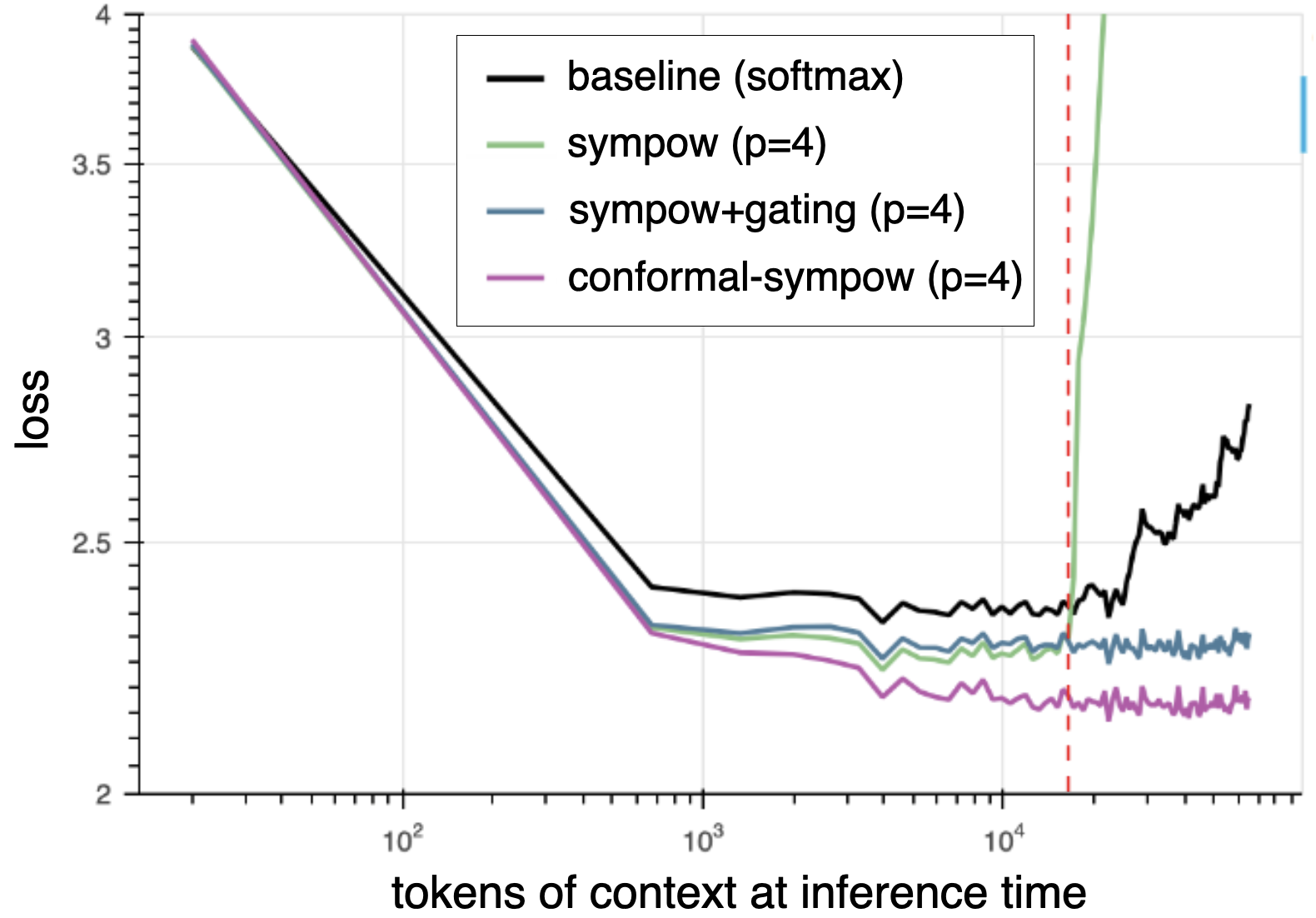}
        \caption*{(b)}
    \end{subfigure}
    \caption{Average loss at different evaluation context lengths ranging from $1$ to $65,536$ tokens. The training context size is $16,384$, indicated by the dashed red line. (a) Sympow is unable to generalize beyond the training context size of $16,384$. (b) Gated sympow generalizes well and conformal-sympow improves performance further.}
    \vspace{-0.3cm}
    \label{fig:sympow_eval}
\end{figure}

\subsection{Does Conformal-Sympow Close the Performance Gap?}
In this subsection, we determine whether conformal-sympow overcomes the limitations of sympow transformers. We implement conformal-sympow in its attention formulation which consists of parameters $W_\gamma$ and $W_\beta$ in each attention layer to learn data-dependent discounts and rotation scaling, as described in Sections~\ref{sec:gating} and~\ref{sec:rotary}. %
We repeat the same experiments as in the previous subsection, scaling the training context sizes from $1024$ to $65,536$ when using sympow with $p=4$ and from $256$ to $8,192$ when using sympow with $p=2$. 

Our results in Figure~\ref{fig:conformal_sympow_vs_softmax_train} demonstrate that conformal-sympow does not suffer from degraded performance compared to the softmax baseline as the training context size grows. When evaluating conformal-sympow on held-out data, we find that it generalizes beyond the training context size, as shown in Figure~\ref{fig:sympow_eval}(b). 

\subsubsection{How important is learning rotary embeddings?}

Our conformal-sympow architecture has two components: data-dependent gating and data-dependent rotary embeddings. Data-dependent gating has been shown in prior work to improve the performance of linear transformers~\citep{yang2023gated}. We isolate the addition of data-dependent rotary embeddings to determine whether scaling rotations is helpful in improving performance. Our results in Figure~\ref{fig:conformal_sympow_vs_softmax_train} and Figure~\ref{fig:sympow_eval}(b) demonstrate that conformal-sympow further improves both training and evaluation performance over sympow with only learned gating (sympow+gating).

\subsection{Related Work}\label{sec:related_work}
The challenge of scaling transformers to long training and evaluation contexts while maintaining computational efficiency has inspired a rich body of research. This section highlights relevant work on gating mechanisms and positional embeddings, both of which play key roles towards computationally efficient context scaling.

\textbf{Gated Linear Attention.} In architectures with linear attention, gating serves to erase information from the finite recurrent state, freeing up memory. Gating has been shown to improve performance of linear transformers, both at training time and at inference time when extrapolating to long evaluation contexts up to $10$ times the train context size~\citep{yang2023gated}. The design of gating mechanisms for linear attention should ensure compatibility with both attention-based and recurrent formulations, preserving the computational efficiency inherent to linear transformers. To this end, several types of gating have been proposed, including using a global, non-data-dependent decay factor~\citep{qin2023scaling}, per head non-data-dependent  decay factors~\citep{sun2023retentive}, scalar and per-head data-dependent decay factors~\citep{dao2024transformers,peng2021random,sun2024you,beck2024xlstm}, and a combination of a non-data-dependent matrix  with a data-dependent vector~\citep{gu2023mamba}. In this work, we adopt the scalar data-dependent gating mechanism introduced in~\citet{dao2024transformers}, as it strikes a balance between expressiveness and efficiency, making it particularly well-suited for integration with sympow transformers. Our results demonstrate that scalar data-dependent gating provides significant benefits to sympow transformers, just as it has for other linear transformer architectures.

\textbf{Positional Embeddings.} In this paper, we adopt rotary embeddings as positional encodings for sympow transformers, leveraging their intrinsic connection to conformal transformations. Rotary embeddings enhance extrapolation performance in softmax transformers compared to sinusoidal position embeddings, which are fixed vectors added to token embeddings before the first transformer layer~\citep{vaswani2017attention,su2024roformer,peng2021random}. An alternative approach, the T5 bias model, replaces rotary embeddings with a learned, shared bias added to each query-key score, conditioned on the distance between the query and key~\citep{raffel2020exploring}. While \citet{peng2021random} demonstrate that this method improves extrapolation performance, it incurs a higher computational cost. In contrast, Attention with Linear Biases (ALiBi) offers further enhancements in extrapolation performance while reducing computational overhead~\citep{peng2021random}. ALiBi biases query-key attention scores with a penalty that is proportional to the distance between the query and key. In Appendix~\ref{app:gating_alibi}, we demonstrate that ALiBi can be interpreted as a form of scalar gating, further underscoring the versatility of gating mechanisms in transformer architectures. Recent work studying the effect of positional embeddings on length generalization suggests that positional embeddings are not essential for inference time extrapolation when using softmax transformers~\citep{kazemnejad2024impact}.

\section{Conclusion}
We present the conformal-sympow transformer, which addresses the shortcomings of sympow transformers by achieving robust performance across scaled training and evaluation contexts. We further verify the compatibility of learned gating and rotary embeddings with sympow transformers in both the attention and recurrent formulations. While we focus on a specific instantiation of learned conformal transformations, we do not extensively explore alternative gating strategies---such as fixed, learned but data-independent, or data-dependent vector gating---nor the broader landscape of learning rotary embeddings. A more comprehensive investigation of these design choices is left for future work.

\section*{Acknowledgments}
We would like to thank Txus Bach and Jono Ridgway for insightful discussions and  Anmol Kagrecha and Wanqiao Xu for valuable feedback on an early version of the paper.

\bibliography{iclr2025_conference}
\bibliographystyle{iclr2025_conference}

\newpage

\appendix
\section{Appendix}
\subsection{Equivalence between Gating and ALiBi}\label{app:gating_alibi}

\citet{press2021train} proposes Attention with Linear Biases (ALiBi), a type of positional encoding that significantly improves the ability of softmax transformers to extrapolate to evaluation contexts longer than the training context size. ALiBi biases query-key attention scores with a penalty that is proportional to the distance between the query and key. We now show that ALiBi is equivalent to applying scalar gating.

In a softmax transformer, the attention scores are computed as
$$
A_{ij}  = \frac{B_{ij}}{\sum_{k=1}^i B_{ik}} \qquad B_{ij} = \text{exp}(Q_i^T K_j)
 \qquad \text{(softmax)}
$$
Recall that we refer to $A_{ij}$ as the attention scores and $B_{ij}$ as the pre-attention scores.

The pre-attention scores after applying ALiBi are
$$
B_{ij} = \text{exp}(Q_i^T K_j + m(j -i))
 \qquad \text{(softmax + ALiBi)}
$$
where $0 < m < 1$ is a head-specific value that is fixed before training. 

Note that 
$$\text{exp}(Q_i^T K_j + m(j - i)) = \gamma^{(i - j)} \text{exp}(Q_i^T K_j)$$
where $\gamma = \text{exp}(-m)$. Since $-m < 0$, $0 < \gamma < 1$. Thus, the application of ALiBi is equivalent to applying scalar gating.

\subsection{Derivations}\label{app:derivations}

\begin{proposition*}
Given a vector of angles $\theta = (\theta_1, \theta_2, ..., \theta_{\frac{d}{2}}$, the block-diagonal rotation matrix
\[
R(\theta) = \begin{pmatrix}
\cos(\theta_1) & -\sin(\theta_1) & \cdots & 0 & 0 \\
\sin(\theta_1) & \cos(\theta_1)  & \cdots & 0 & 0 \\
\vdots & \vdots & \ddots & \vdots & \vdots \\
0 & 0 & \cdots & \cos(\theta_{d/2}) & -\sin(\theta_{d/2}) \\
0 & 0 & \cdots & \sin(\theta_{d/2}) & \cos(\theta_{d/2})
\end{pmatrix},
\]
satisfies \( R(\theta)^k = R(k \theta) \) for any positive integer \( k \).
\end{proposition*}

\begin{proof}
We prove this statement by induction on \( k \) for a single \( 2 \times 2 \) rotation matrix \( R(\theta_i) \), and then extend it to the full block diagonal matrix.

For the base case \( k = 1 \), we have:
\[
R(\theta_i)^1 = R(\theta_i),
\]
which is equivalent to \( R(1 \cdot \theta_i) = R(\theta_i) \). Thus, the base case holds.

Assume that for some positive integer \( k \), the property holds:
\[
R(\theta_i)^k = R(k \theta_i).
\]

We need to show that \( R(\theta_i)^{k+1} = R((k+1)\theta_i) \). Using the definition of matrix exponentiation:
\[
R(\theta_i)^{k+1} = R(\theta_i) R(\theta_i)^k.
\]

By the inductive hypothesis, \( R(\theta_i)^k = R(k\theta_i) \). Substituting this:
\[
R(\theta_i)^{k+1} = R(\theta_i) R(k \theta_i).
\]

The product of two rotation matrices corresponds to a rotation by the sum of their angles. Therefore:
\[
R(\theta_i) R(k \theta_i) = R(\theta_i + (k \theta_i)) = R((k+1)\theta_i).
\]

Thus, \( R(\theta_i)^{k+1} = R((k+1)\theta_i) \), completing the inductive step. By induction, the property holds for all \( k \geq 1 \).

\textbf{Extension to Block Diagonal Matrices}

Consider the block diagonal matrix \( R(\theta) \), where:
\[
R(\theta) = \begin{pmatrix}
R(\theta_1) & 0 & \cdots & 0 \\
0 & R(\theta_2) & \cdots & 0 \\
\vdots & \vdots & \ddots & \vdots \\
0 & 0 & \cdots & R(\theta_{d/2})
\end{pmatrix}.
\]

Since each block \( R(\theta_i) \) is independent of the others, the \( k \)-th power of \( R(\theta) \) is the block diagonal matrix with each block raised to the \( k \)-th power:
\[
R(\theta)^k = \begin{pmatrix}
R(\theta_1)^k & 0 & \cdots & 0 \\
0 & R(\theta_2)^k & \cdots & 0 \\
\vdots & \vdots & \ddots & \vdots \\
0 & 0 & \cdots & R(\theta_{d/2})^k
\end{pmatrix}.
\]

Using the result for a single rotation matrix, \( R(\theta_i)^k = R(k\theta_i) \), we get:
\[
R(\theta)^k = \begin{pmatrix}
R(k \theta_1) & 0 & \cdots & 0 \\
0 & R(k \theta_2) & \cdots & 0 \\
\vdots & \vdots & \ddots & \vdots \\
0 & 0 & \cdots & R(k \theta_{d/2})
\end{pmatrix}.
\]

This is equivalent to the block diagonal matrix \( R(k\theta) \), where \( k\theta = (k\theta_1, k\theta_2, \dots, k\theta_{d/2}) \).

Thus, by induction and the block diagonal structure, \( R(\theta)^k = R(k\theta) \) for any positive integer \( k \).
\end{proof}

\newpage

\rotaryequivalence*

\begin{proof}

We begin by writing the output $Y_i$ at time step $i$ in the attention formulation.
For notational simplicity, let $C_i = \sum_{k=1}^i (Q_i'^T K_k')^p = \sum_{k=1}^i \phi^p(Q_i')^T \phi^p(K_k')$.

\begin{align*}
Y_i &= \sum_{j=1}^i \frac{\left( Q_i^T (R^{i-j})^T K_j\right) ^p V_j}{C_i} \\
    &= \sum_{j=1}^i \frac{\left( Q_i^T (R^i)^T R^j K_j\right)^p V_j}{C_i} \\
    &= \sum_{j=1}^i \frac{\left(Q_i^T R(\mu_i)^T R(\mu_j) K_j\right)^p V_j}{C_i} \\
    &= \sum_{j=1}^i \frac{\left((R(\mu_i) Q_i)^T R(\mu_j) K_j\right)^p V_j}{C_i} \\
    &= \sum_{j=1}^i \frac{\left(\phi^p(R(\mu_i) Q_i)^T \phi^p(R(\mu_j) K_j)\right) V_j}{C_i} \\
    &= \sum_{j=1}^i \frac{V_j \phi^p(R(\mu_j) K_j)^T \phi^p(R(\mu_i) Q_i)}{C_i} \\
    &= \frac{\left( \sum_{j=1}^i V_j \phi^p(R(\mu_j) K_j)^T \right)  \phi^p(R(\mu_i) Q_i) }{C_i} \\
    &= \frac{S_i  \phi^p(R(\mu_i) Q_i) }{Z_i \phi^p(R(\mu_i) Q_i)}
\end{align*}

which is the recurrent formulation of the output $Y_i$. The last line above uses the fact that 

$$
\sum_{j=1}^i V_j \phi^p(R(\mu_j) K_j)^T = S_i 
$$

and

\begin{align*}
C_i &= \sum_{k=1}^i \phi^p(Q_i')^T \phi^p(K_k') \\
    &= \sum_{k=1}^i \phi^p(R(\mu_i) Q_i)^T \phi^p(R(\mu_k) K_k) \\
    &= \sum_{k=1}^i \phi^p(R(\mu_k) K_k)^T \phi^p(R(\mu_i) Q_i) \\
    &= Z_i \phi^p(R(\mu_i) Q_i)
\end{align*}
\end{proof}

\newpage

\gatingequivalence*
\begin{proof}
We begin by writing the output $Y_i$ at time step $i$ in the attention formulation.
For notational simplicity, let $C_i = \sum_{k=1}^i b_{ik} (Q_i'^T K_k')^p = \sum_{k=1}^i b_{ik} \phi^p(Q_i')^T \phi^p(K_k')$.

\begin{align*}
    Y_i &= \sum_{j=1}^i \frac{b_{ij} \left(Q_i'^T K'_j\right)^p V_j}{C_i} \\
    &= \sum_{j=1}^i \frac{b_{ij} \left(\phi^p(Q'_i)^T \phi^p(K'_j)\right) V_j}{C_i} \\
    &= \sum_{j=1}^i \frac{b_{ij} V_j \phi^p(K'_j)^T \phi^p(Q'_i)}{C_i} \\
    &= \frac{\left( \sum_{j=1}^i b_{ij} V_j \phi^p(K'_j)^T \right)  \phi^p(Q'_i) }{C_i} \\
    &= \frac{S_i  \phi^p(Q'_i) }{Z_i \phi^p(Q'_i)}
\end{align*}

which is the recurrent formulation of the output $Y_i$. The last line above uses the fact that $\sum_{j=1}^i b_{ij} V_j \phi^p(K'_j)^T = S_i$ and

\begin{align*}
C_i &= \sum_{k=1}^i b_{ik} \phi^p(Q_i')^T \phi^p(K_k') \\
    &= \sum_{k=1}^i b_{ik} \phi^p(K_k')^T \phi^p(Q_i') \\
    &= Z_i \phi^p(Q'_i)
\end{align*}

We prove that $S_i = \sum_{j=1}^i b_{ij} V_j \phi^p(K'_j)^T$ by induction.  As the base case, note that 
$S_1 =  V_1 \phi^p(K'_1)^T$. For the inductive step, suppose that for $k > 1$, $S_k = \sum_{j=1}^k b_{kj} V_j \phi^p(K'_j)^T$.
Then 

\begin{align*}
S_{k+1} &= \gamma_{k+1} S_k + V_{k+1}\phi^p(K'_{k+1})^T \\
        &= \gamma_{k+1} \left( \sum_{j=1}^k b_{kj} V_j \phi^p(K'_j)^T \right) + V_{k+1}\phi^p(K'_{k+1})^T \\
        &= \left( \sum_{j=1}^k b_{(k+1)j} V_j \phi^p(K'_j)^T \right) + V_{k+1}\phi^p(K'_{k+1})^T \\
        &= \sum_{j=1}^{k+1} b_{(k+1)j} V_j \phi^p(K'_j)^T
\end{align*}

This completes the inductive step.
\end{proof}

\newpage

\conformalequivalence*
\begin{proof}
Note that the symmetric power embedding function is equivalent to applying the tensor product and removing redundant information resulting from symmetry. For mathematical simplicity, we prove the corresponding result for which the embedding function is the repeated tensor product $\otimes^p$. The corresponding proposition is stated below.

Let $V$ be a vector space with dimension $d$ and basis vectors $\{ v_1, v_2, \dots, v_d \}$, and let $P \in \mathbb{R}^{d \times d}$ be a rotation matrix. Define the linear map $\bar{P} : V^{\otimes p} \to V^{\otimes p}$ (the tensor product of $p$ copies of $V$) by its action on the basis elements as
$$
\bar{P}(v_{i_1} \otimes v_{i_2} \otimes \dots \otimes v_{i_p}) = (P v_{i_1}) \otimes (P v_{i_2}) \otimes \dots \otimes (P v_{i_p}) \quad \text{for all } i_1, i_2, \dots, i_p.
$$
Then $\bar{P} \in \mathbb{R}^{d^p \times d^p}$ is a rotation matrix.

We need to show that $\bar{P}$ satisfies the properties of a rotation matrix, namely:
1. $\bar{P}$ is an orthogonal matrix, i.e., $\bar{P}^T \bar{P} = I$.
2. $\det(\bar{P}) = 1$, so that $\bar{P}$ represents a proper rotation.

\textbf{Step 1: Orthogonality of $\bar{P}$.}

Since $\bar{P}$ is defined by its action on the basis elements of $V^{\otimes k}$ as
$$
\bar{P}(v_{i_1} \otimes v_{i_2} \otimes \dots \otimes v_{i_p}) = (P v_{i_1}) \otimes (P v_{i_2}) \otimes \dots \otimes (P v_{i_p}),
$$
and $P$ is an orthogonal matrix, i.e., $P^T P = I_d$, where $I_d$ is the identity matrix in $\mathbb{R}^{d}$, we need to verify that $\bar{P}$ preserves the inner product in the tensor product space.
The inner product of two basis elements $v_{i_1} \otimes v_{i_2} \otimes \dots \otimes v_{i_p}$ and $v_{j_1} \otimes v_{j_2} \otimes \dots \otimes v_{j_p}$ in $V^{\otimes p}$ is given by:
$$
\langle v_{i_1} \otimes v_{i_2} \otimes \dots \otimes v_{i_p}, v_{j_1} \otimes v_{j_2} \otimes \dots \otimes v_{j_k} \rangle = \prod_{p=1}^{p} \langle v_{i_p}, v_{j_p} \rangle.
$$
Applying $\bar{P}$ to this inner product, we get:
$$
\langle \bar{P}(v_{i_1} \otimes \dots \otimes v_{i_p}), \bar{P}(v_{j_1} \otimes \dots \otimes v_{j_p}) \rangle = \prod_{p=1}^{p} \langle P v_{i_p}, P v_{j_p} \rangle.
$$
Since $P$ is orthogonal, we have $\langle P v_{i_p}, P v_{j_p} \rangle = \langle v_{i_p}, v_{j_p} \rangle$ for each $p$. Therefore, $\bar{P}$ preserves the inner product, meaning that $\bar{P}$ is an orthogonal matrix, i.e., $\bar{P}^T \bar{P} = I_{d^p}$.

\textbf{Step 2: Determinant of $\bar{P}$.}

Next, we show that $\det(\bar{P}) = 1$. Since $\bar{P} = P \otimes P \otimes \dots \otimes P$ (a $p$-fold tensor product of $P$ with itself), we can use the property of the determinant for tensor products of matrices. Specifically, if $A$ and $B$ are square matrices, then:
$$
\det(A \otimes B) = \det(A)^{\dim(B)} \det(B)^{\dim(A)}.
$$
In our case, since $\bar{P} = P \otimes P \otimes \dots \otimes P$, we have:
$$
\det(\bar{P}) = \det(P)^{p \cdot d}.
$$
Since $P$ is a rotation matrix in $\mathbb{R}^d$, we know that $\det(P) = 1$. Therefore:
$$
\det(\bar{P}) = 1^{p \cdot d} = 1.
$$
Thus, $\bar{P}$ is a proper rotation matrix.
\end{proof}

\newpage
\subsection{Experiments}\label{app:experiments}

\subsubsection{Training curves}
We present full training curves of sympow, sympow with gating, and conformal-sympow with $p=4$ (Figure~\ref{fig:training_curves_p=4}) and $p=2$ (Figure~\ref{fig:training_curves_p=2}). We can see that both sympow+gating and conformal-sympow improve optimization over sympow throughout training.

\subsection{Compute Overhead of Data-Dependent Gating and Data-Dependent Rotary Embeddings}
Adding data-dependent gating increases the total parameter count by $0.09\%$. Together, adding data-dependent gating and data-dependent rotary embeddings increases the parameter count by $0.18\%$. 

Below, we discuss the compute overhead of adding data-dependent gating and data-dependent rotary embeddings in a single attention head under the recurrent formulation of conformal-sympow. 

To understand the compute overhead of adding data-dependent gating, we first compute the number of operations in a single recurrent state update. Here, we assume a batch size of 1. Consider a recurrent state with dimension $O(d \times D)$. Without gating, the recurrent state update has $O(dD)$ operations. Gating computes a single scalar discount value by taking a dot product between the input embedding $X_i \in \mathbb{R}^d$ and the parameter $W_\gamma$, which adds $O(d)$ operations. (Recall from Section~\ref{sec:sympow_detailed} that $d < D < d^p$). Multiplying $\gamma$ with the recurrent state adds $O(dD)$ operations. The total number of operations is still $O(dD)$ which is significantly cheaper than the compute cost of the attention formulation. The attention formulation compute cost is quadratic in the sequence length, which is much larger than $dD$ (when $p$ is small as in this paper) when the sequence length is large.

To understand the compute overhead of adding data-dependent rotary embeddings, note that computing the scalar $\beta_i$ by which we scale the fixed rotation rates $\theta$ involves taking a dot product between the input embedding $X_i$ and $W_\beta$, which adds $O(d)$ operations. Adding the resulting scalar to the cumulative scalar adds $O(1)$ operations. Finally, multiplying with $\theta$ adds $O(d)$ operations. The rest of the computation is the same as without data-dependent rotary embeddings. The total number of operations is again still $O(dD)$.

\begin{figure}[t] %
    \centering
    \begin{subfigure}[b]{0.45\textwidth}
        \centering
        \includegraphics[width=\textwidth]{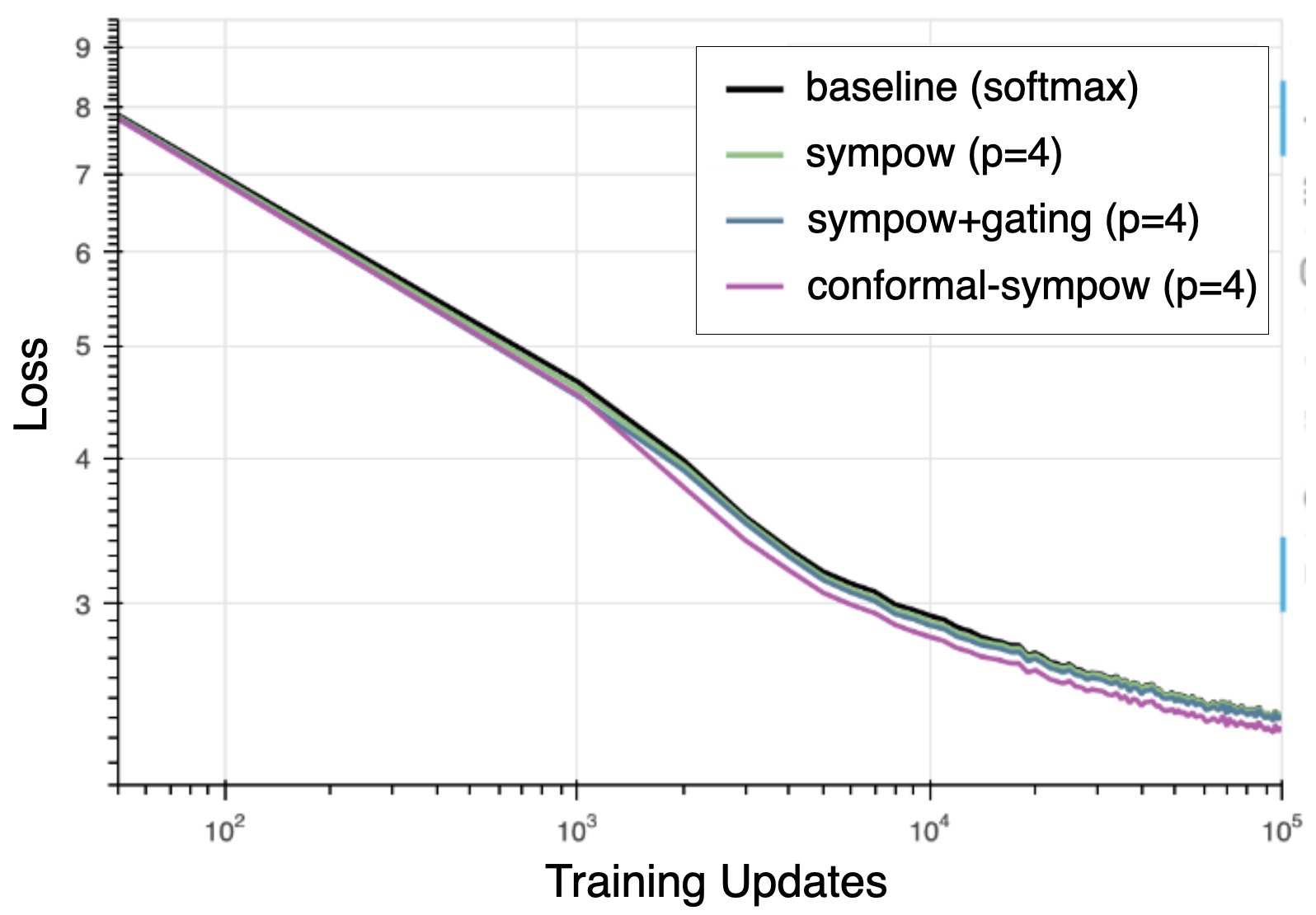}
        \caption*{(a) context size = $1024$}
    \end{subfigure}
    \hfill
    \begin{subfigure}[b]{0.45\textwidth}
        \centering
        \includegraphics[width=\textwidth]{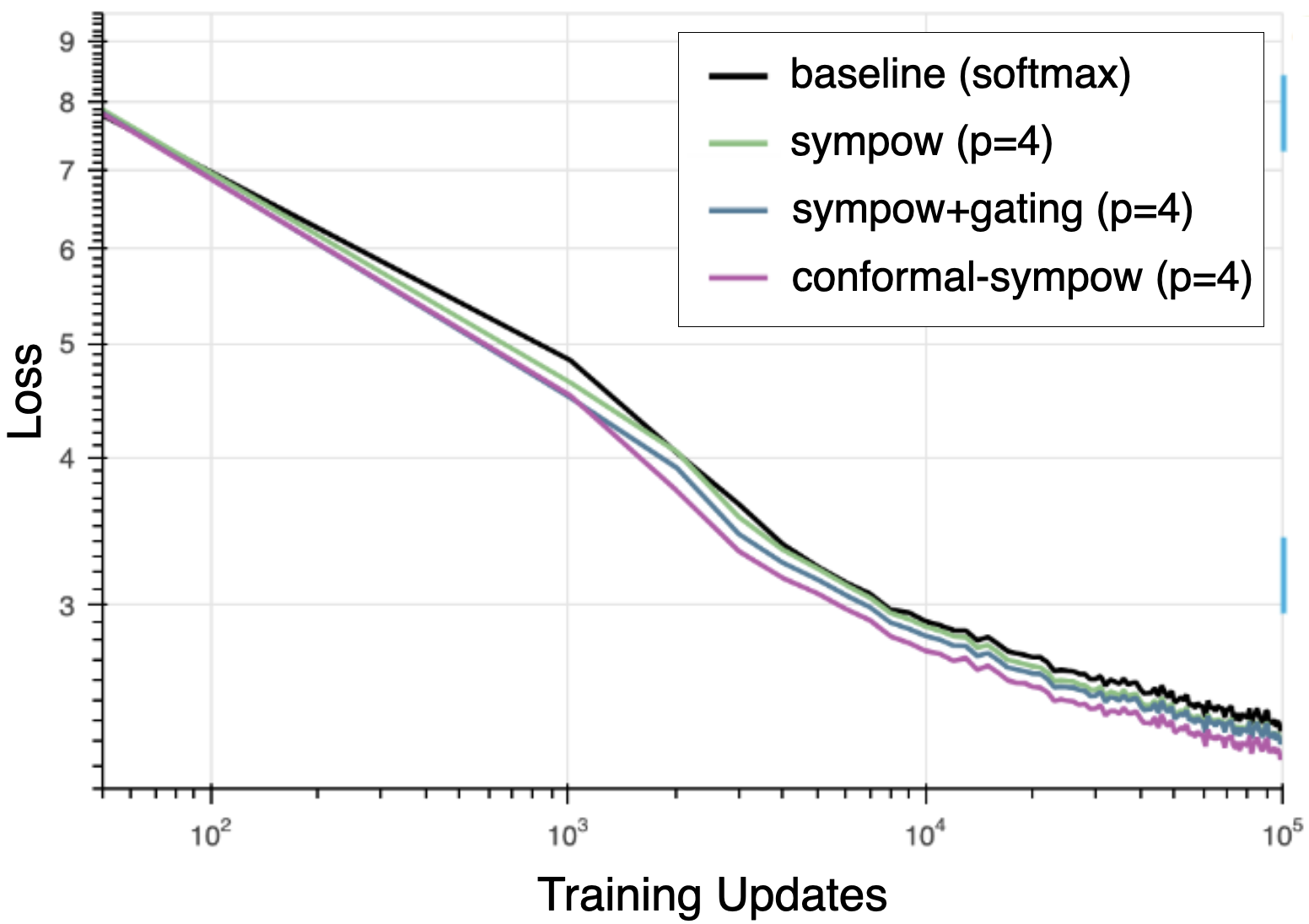}
        \caption*{(b) context size = $4096$}
    \end{subfigure}
    
    \vspace{0.5cm} %
    
    \begin{subfigure}[b]{0.45\textwidth}
        \centering
        \includegraphics[width=\textwidth]{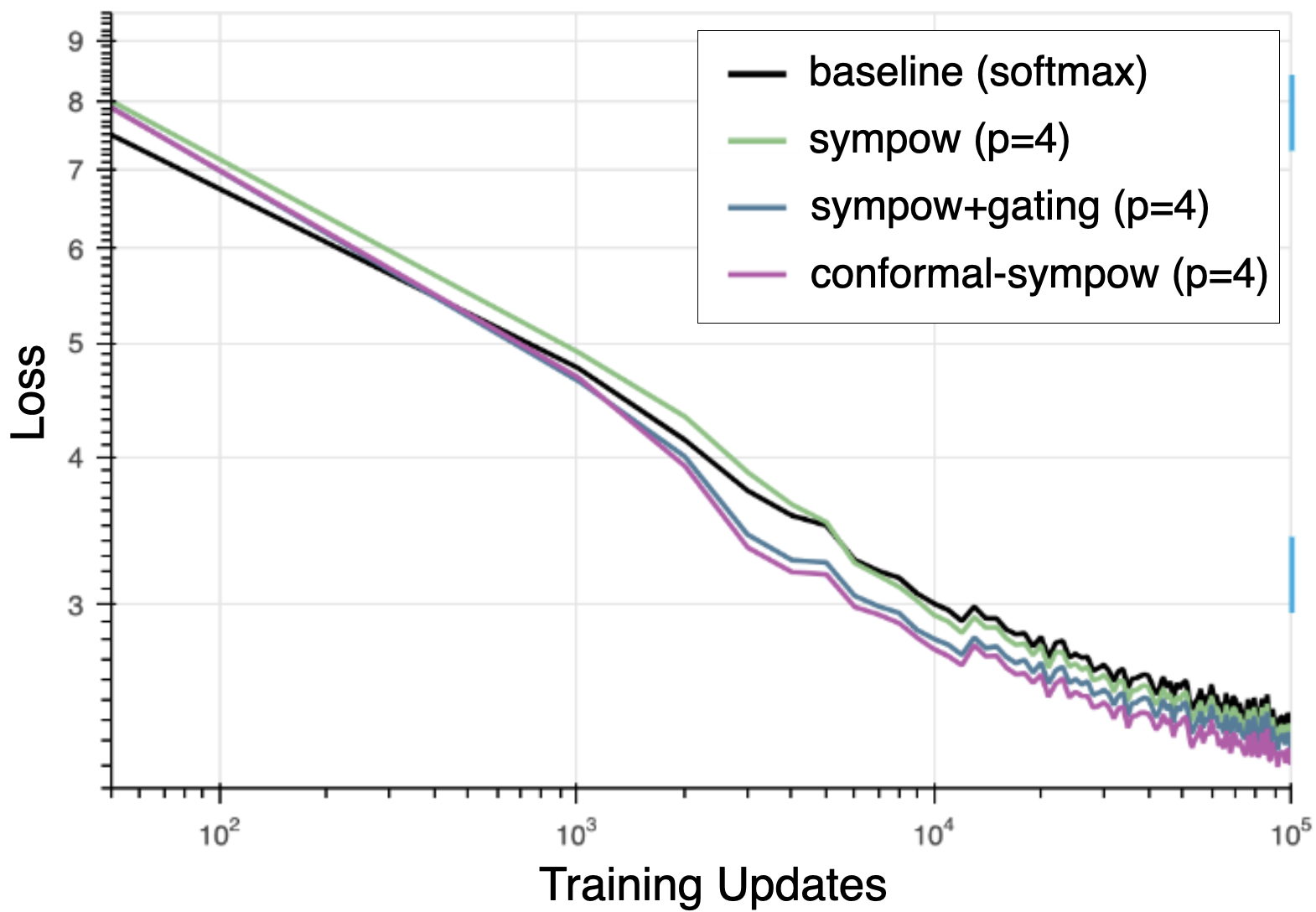}
        \caption*{(c) context size = $16384$}
    \end{subfigure}
    \hfill
    \begin{subfigure}[b]{0.45\textwidth}
        \centering
        \includegraphics[width=\textwidth]{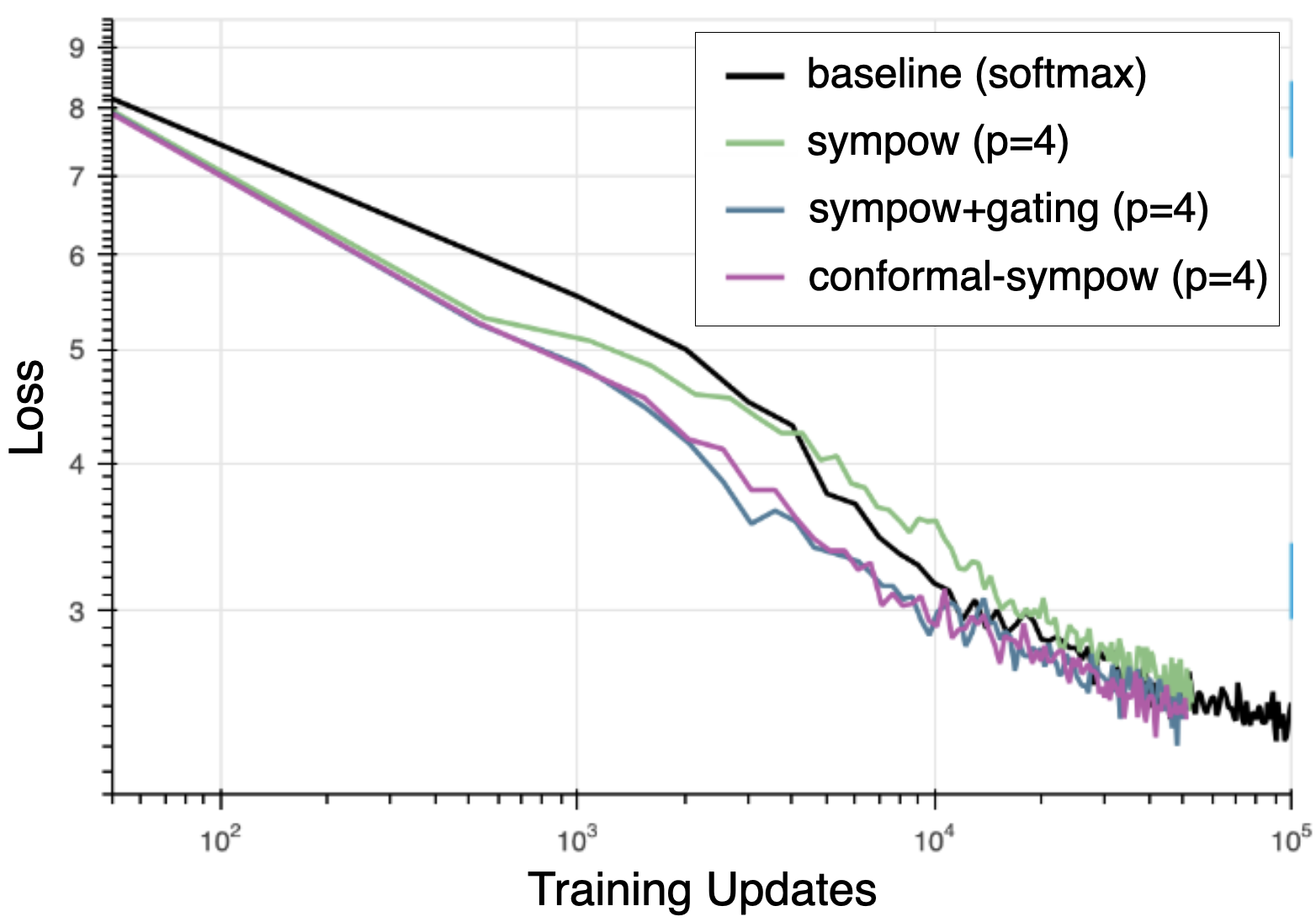}
        \caption*{(d) context size = $65536$ }
    \end{subfigure}

    \caption{Training curves for sympow, sympow+gating, and conformal-sympow with $p=4$ at different training context lengths. We can see that both sympow+gating and conformal-sympow improve optimization over sympow throughout training.}
    \vspace{-0.3cm}
    \label{fig:training_curves_p=4}
\end{figure}

\begin{figure}[t] %
    \centering
    \begin{subfigure}[b]{0.45\textwidth}
        \centering
        \includegraphics[width=\textwidth]{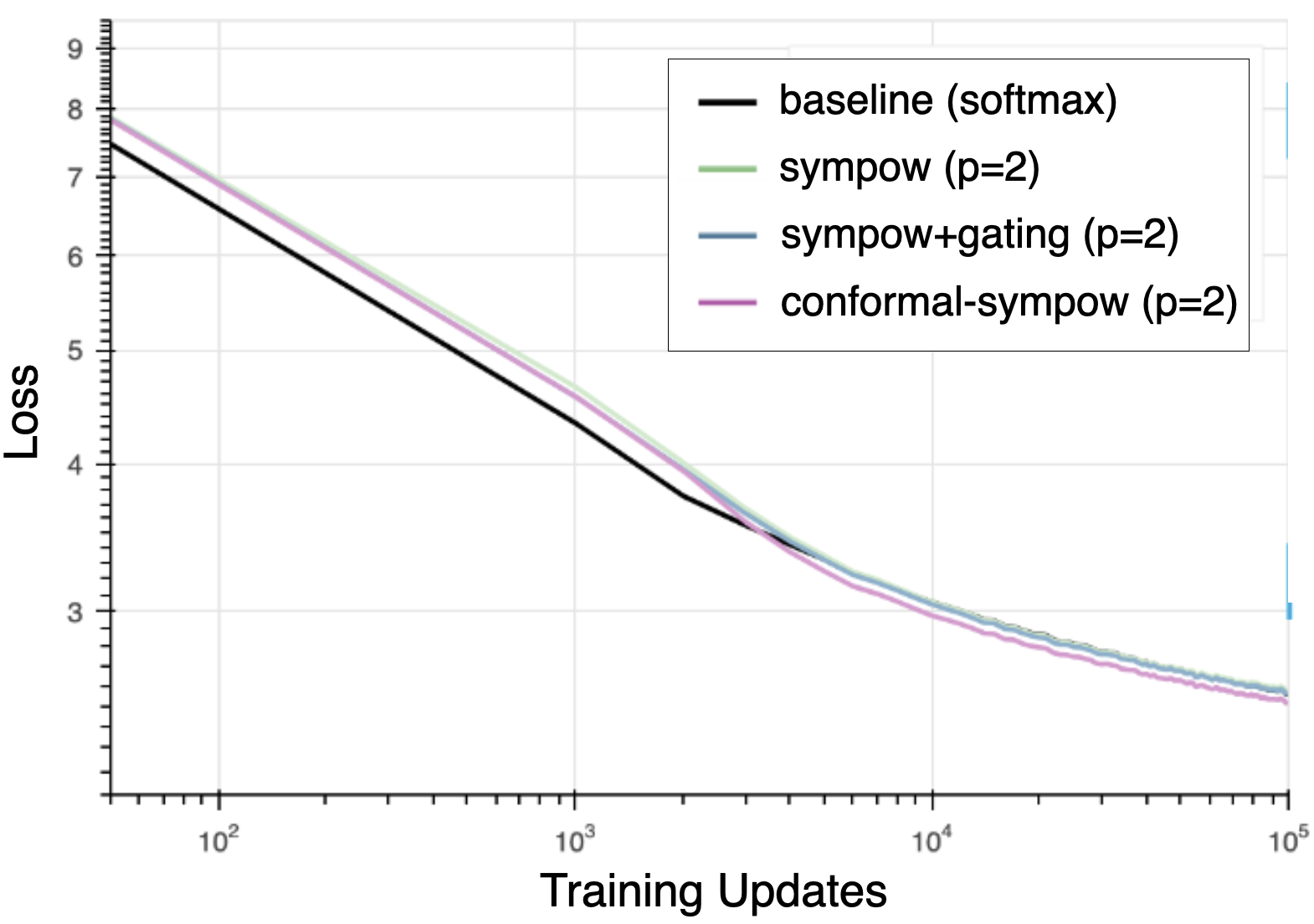}
        \caption*{(a) context size = $256$}
    \end{subfigure}
    \hfill
    \begin{subfigure}[b]{0.45\textwidth}
        \centering
        \includegraphics[width=\textwidth]{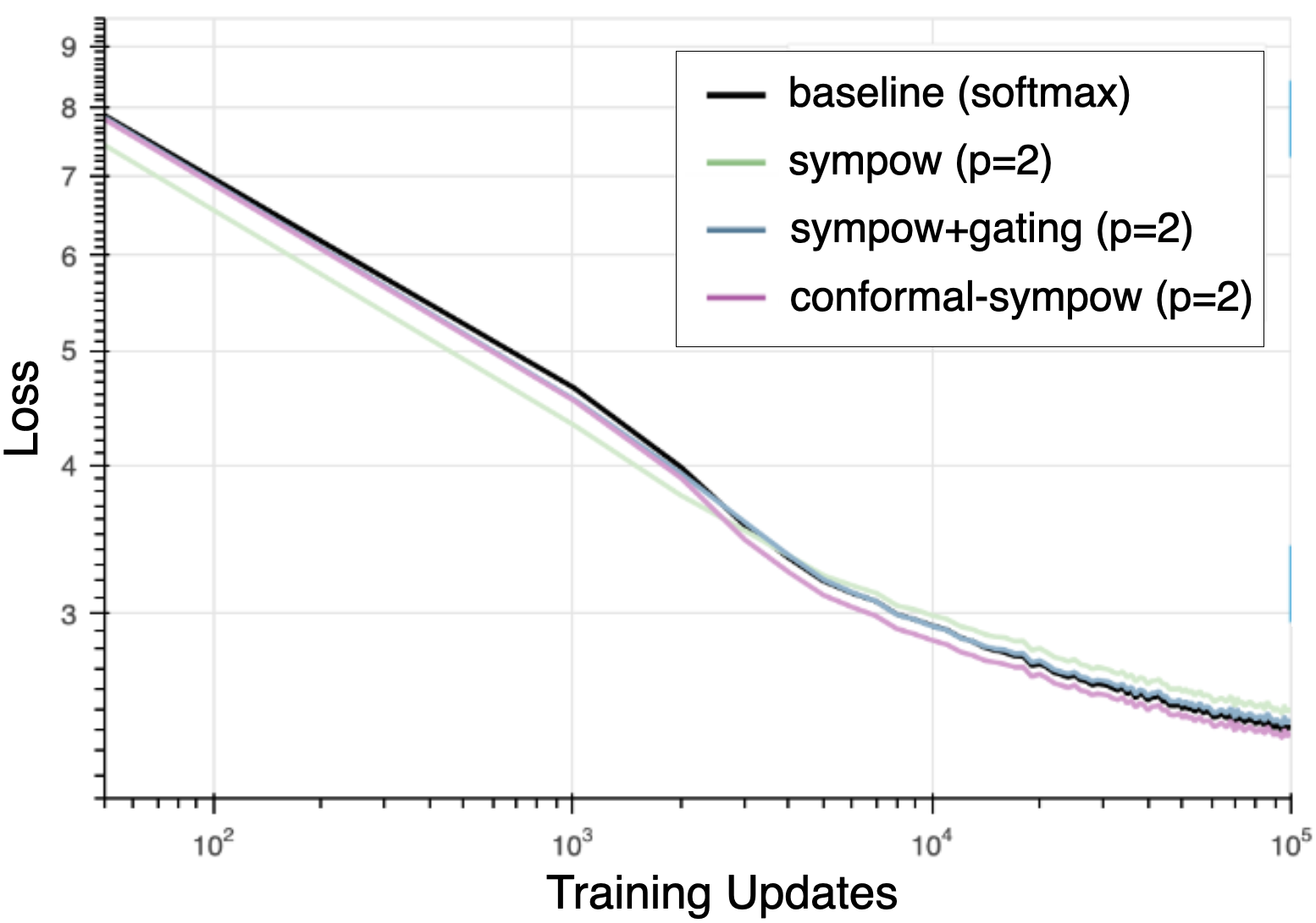}
        \caption*{(b) context size = $1024$}
    \end{subfigure}
    
    \vspace{0.5cm} %
    
    \begin{subfigure}[b]{0.45\textwidth}
        \centering
        \includegraphics[width=\textwidth]{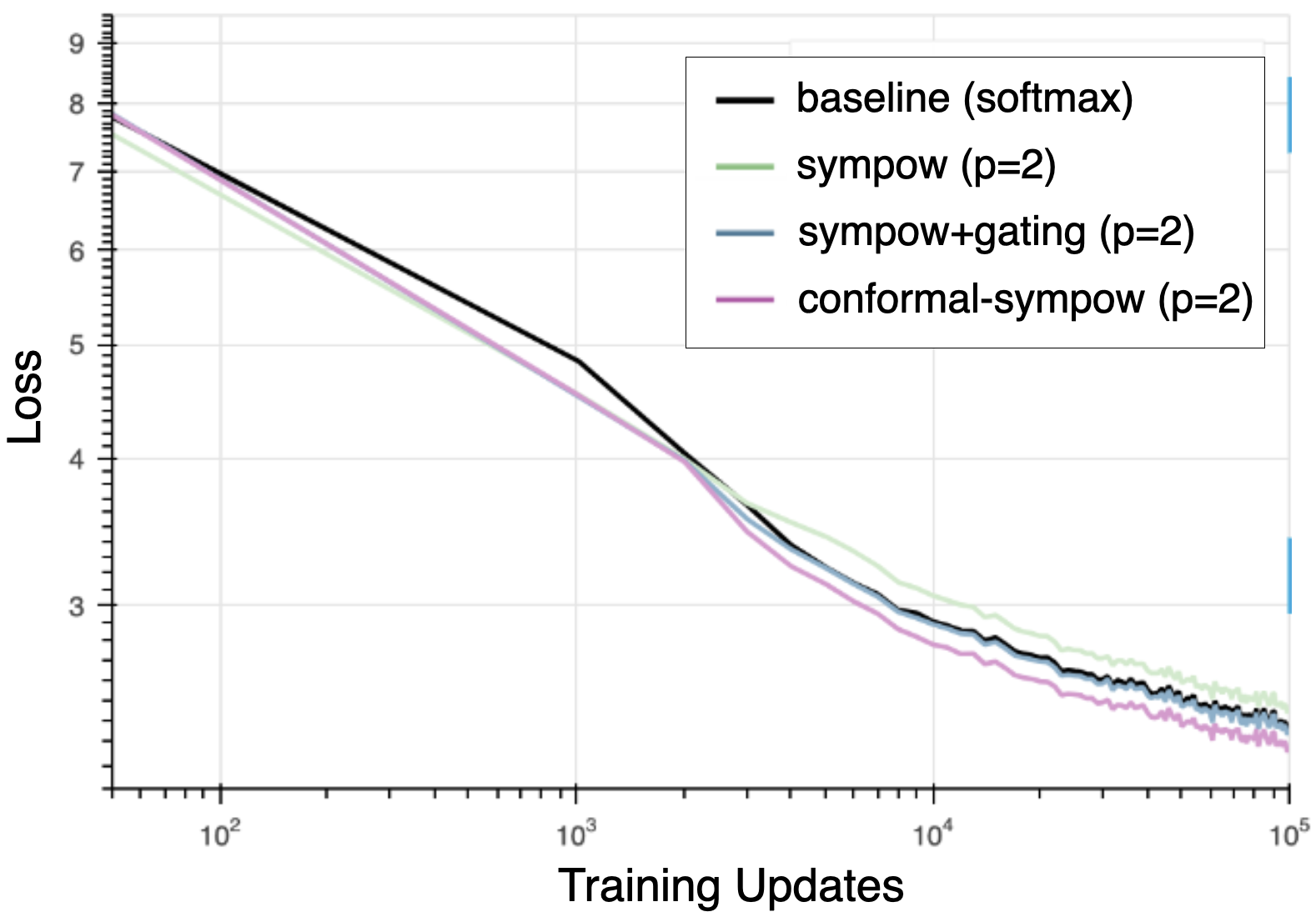}
        \caption*{(c) context size = $4096$}
    \end{subfigure}
    \hfill
    \begin{subfigure}[b]{0.45\textwidth}
        \centering
        \includegraphics[width=\textwidth]{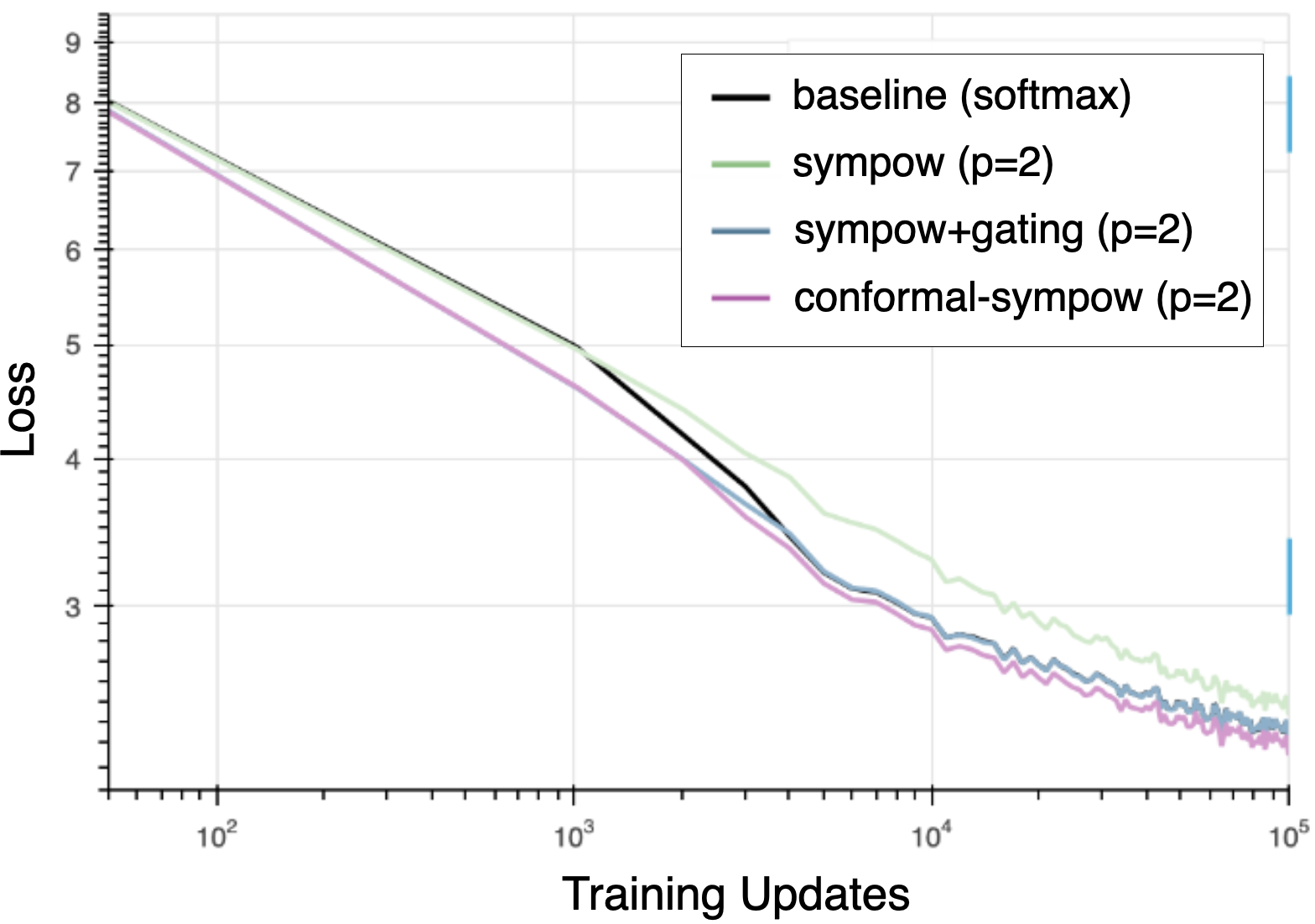}
        \caption*{(d) context size = $8192$ }
    \end{subfigure}

    \caption{Training curves for sympow, sympow+gating, and conformal-sympow with $p=2$ at different training context lengths. We can see that both sympow+gating and conformal-sympow improve optimization over sympow throughout training.}
    \vspace{-0.3cm}
    \label{fig:training_curves_p=2}
\end{figure}

\end{document}